\documentclass[review]{elsarticle}

\usepackage{lineno,hyperref}
\modulolinenumbers[5]

\makeatletter
\let\@date\@empty 
\makeatother
\usepackage{makecell}
\usepackage{helvet}
\usepackage{courier}
\usepackage{graphicx}
\usepackage{amsmath}
\usepackage{amsthm}
\usepackage{amssymb}
\usepackage{dutchcal}
\usepackage{array}
\usepackage{subcaption}
\usepackage{caption}
\usepackage{natbib}
\usepackage{tikz}
\usetikzlibrary{positioning, shapes.geometric}
\usetikzlibrary{shapes,arrows}
\tikzstyle{startstop} = [rectangle,rounded corners, minimum width=4cm,minimum height=0.7cm,text centered,text width =4.5cm, draw=black]
\tikzstyle{io} = [trapezium, trapezium left angle = 70,trapezium right angle=110,minimum width=4cm,minimum height=0.5cm,text centered,text width =4.5cm,draw=black]
\tikzstyle{process} = [rectangle,minimum width=4cm,minimum height=0.5cm,text centered,text width =4.5cm,draw=black]
\tikzstyle{decision} = [diamond,aspect = 3,text centered,draw=black]
\tikzstyle{arrow} = [thick,->,>=stealth]
\tikzstyle{straightline} = [line width = 1pt,-]
\tikzstyle{point}=[coordinate]
\usepackage[ruled]{algorithm2e}
\usepackage{algpseudocode}
\usepackage{multirow} 
\usepackage{pifont}

\usepackage{comment}
\usepackage{amsmath,amssymb} 
\usepackage{color}
\usepackage{graphicx}
\usepackage{amsfonts,pifont,enumerate,cases}
\usepackage{mathrsfs}

\usepackage{verbatim}
\usepackage{bm}
\usepackage{latexsym}
\usepackage{array}
\usepackage{multirow}
\usepackage{booktabs}

\newtheorem{theorem}{Theorem}

\usepackage{footmisc}
\usepackage{epstopdf}
\DeclareGraphicsExtensions{.eps}

\begin{document}

\begin{frontmatter}

\title{Manifold Clustering with Schatten p-norm Maximization}


\author[mymainaddress]{Fangfang Li}
\author[mymainaddress]{Quanxue Gao}

\address[mymainaddress]{School of Telecommunications Engineering, Xidian University, Xi'an 710071, China.}

\begin{abstract}
	Manifold clustering, with its exceptional ability to capture complex data structures, holds a pivotal position in cluster analysis. However, existing methods often focus only on finding the optimal combination between K-means and manifold learning, and overlooking the consistency between the data structure and labels. To address this issue, we deeply explore the relationship between K-means and manifold learning, and on this basis, fuse them to develop a new clustering framework. Specifically, the algorithm uses labels to guide the manifold structure and perform clustering on it, which  ensures the consistency between the data structure and labels. Furthermore, in order to naturally maintain the class balance in the clustering process, we maximize the Schatten p-norm of labels, and provide a theoretical proof to support this. Additionally, our clustering framework is designed to be flexible and compatible with many types of distance functions, which facilitates efficient processing of nonlinear separable data. The experimental results of several databases confirm the superiority of our proposed model.
\end{abstract}

\end{frontmatter}

\section{Introduction}\label{sec:introduction}
Clustering is an unsupervised learning technique that has become an important area in pattern analysis and artificial intelligence research due to the prevalence of unlabeled data in the real world. 
Among all kinds of clustering methods, the distance-based clustering algorithm is the most widely used, which can be divided into global clustering and local clustering.
A typical example of global methods is K-means~\cite{hartigan1979algorithm}, which aims to achieve an optimal clustering by minimizing the sum of squared distances from each data point to its nearest cluster centroid. 
However, because K-means uses Euclidean distances, it may not accurately capture similarities and differences between data when they are presented with nonlinear separability.
To address this, some researchers employ kernel functions—such as the  Euler kernel~\cite{Euler}, and multi-kernel~\cite{SimpleMKKM, 9212617}—to map data into a high-dimensional kernel space where it becomes linearly separable. Nevertheless, even with kernel-based enhancements, the choice of initial cluster centers remains crucial in K-means. Incorrect initial selections can lead the algorithm to converge to local optima, significantly impacting the clustering outcome. To mitigate this issue, various strategies for selecting initial clustering centers have been proposed, such as K-Means++~\cite{11, 10} to reduces randomness, methods based on data point density to consider local density~\cite{CIDP}, partitioning techniques~\cite{fenqu} based on spatial distribution to focus on global distribution, and genetic algorithms~\cite{GA} to optimize high-quality centers. However, these methods still focus on determining the clustering center. Other researchers have explored ways to avoid K-means centroid estimation. 
Although the above K-means and its variants have demonstrated excellent performance in global information mining, they can encounter challenges when faced with complex data structures.

In contrast, local clustering methods based on manifold learning show significant advantages when dealing with data with complex structures.
But manifold learning and clustering in these methods are  conducted separately. To address this limitation, Nie et al. combines distance-based clustering with manifold learning, simultaneously learning the similarity matrix and clustering structure of the data through probabilistic neighborhoods. Wang et al. further adds entropy regularization on this basis, helping the clustering algorithm to better balance information among different neighbors. 
However, they ignore the consistency of the  manifold structure with labels and are limited to finding the best combination of distance and manifold, which limits their effectiveness on complex datasets.

To solve these problems, we propose a new clustering framework that fuses manifold and distance. Specifically, we establish a clear connection between manifolds and label distributions, build manifold structures by using labels, and optimize labels by clustering on it, thus ensuring consistency between manifold structures and labels. 
Notably, in algorithm design, we incorporate the concept of Schatten p-norm, which is widely used in fields such as Image denoising and anomaly detection. In these applications, minimizing the Schatten p-norm is often crucial for achieving the desired outcomes
In this paper, however, we reveal another important and novel application of the Schatten p-norm: By maximizing the Schatten p-norm, class balance can be achieved in the clustering process. 
\begin{itemize}
	\item Fusing  \emph{K}-means with manifold learning within a clustering framework, beyond a mere combination, effectively harnesses both the global architecture and local characteristics of the data.	
	\item Constructing  manifold structure with labels and clustering on it ensures manifold-labels consistency while capturing manifold structure.
	\item We find  a significant role of the Schatten p-norm. By maximizing the Schatten p-norm, our model naturally maintains class balance during the optimization process.
\end{itemize}

\textbf{Notations}:
For the sake of representation, we herein introduce the notations used throughout the paper.
Scalars, vectors, and matrices are represented by lowercase letters such as $q$, bold lower case letters such as ${\mathbf{g}}$ , bold upper case letters such as ${\mathbf{G}}$, respectively. The $i$-th row and $j$-th column of matrix ${\mathbf{G}}$ are used as ${\mathbf{g}}^{i}$ and ${\mathbf{g}}_{j}$, respectively.

\section{Unified Framework}

K-Means is a classical distance-based clustering method that calculates and assigns labels based on the distance between data. However, K-Means primarily focuses on preserving the global geometric structure of the data. In contrast, manifold learning is good at revealing  local structural features  of data, but unfortunately, they cannot directly generate clustering labels. 
In order to combine the advantages of these two techniques, existing methods primarily focus on combining the two types of information. 
This combination method aims to identify the optimal integration of k-means and manifold learning, rather than representing a fundamental optimization update.
Furthermore, they overlook the manifold structure and label consistency of the data. This oversight may lead to mislabeling of samples within the same category or incorrect grouping of samples from different categories.
Therefore, maintaining consistency between sample manifold structure and labels is essential in the clustering process.
We propose Theorem 1, which aims to integrate K-means and manifold learning into a unified clustering framework by exploring their relationship, thereby achieving consistency between the manifold structure  and  labels.

\begin{theorem}\label{theorem1}
	Suppose data matrix ${\mathbf{X}}=[\mathbf{x}_1\textrm{, }\ldots\textrm{ , }\mathbf{x}_N]\in \mathbb{R}^{d\times N}$, where $N$ is the number of samples, $d$ is the feature dimension. The label matrix is defined as
	${{\mathbf{G}}} = {\left[ {{\mathbf{g}}_1, \ldots ,{\mathbf{g}}_K} \right]} \in \mathbb{R}^{N\times K}$
	, where $\mathbf{g}_j$ is the $j$-th column of $\mathbf{G}$. The element $g_{ij} = 1$ if sample $\mathbf{x}_i$ belongs to the $j$-th cluster, and $g_{ij} = 0$ otherwise. Here, $K$ denotes the number of clusters.
	Let ${\mathbf{P}}  =diag(p_1  \textrm{, }\ldots\textrm{ , }p_K  )$,
	${\mathbf{R}}  =diag(r_1  \textrm{, }\ldots\textrm{ , }r_N  )$, where
	$p_j   = \sum\limits_{i} {{g_{ij}  }}$, ${{r_i}  } = \sum\limits_{j}{ {g_{ij}  }}$.
	Then, we have
	\begin{equation}\label{theorem_manifold}
		\begin{aligned}
			{\sum\limits_{i,j}{ {g_{ij}  \left\| {{\mathbf{x}}_i   - {\mathbf{u}}_j  } \right\|_F^2}}} \
			={\sum\limits_{i,l}{ {\left\| {{\mathbf{x}}_i   - {\mathbf{x}}_l  } \right\|_F^2m_{il}  }} }
		\end{aligned}
	\end{equation}
	where the manifold structure ${\mathbf{M}}$ represents the  cluster structure in the data, ${\mathbf{M}  } = {\mathbf{A}  }{{\mathbf{A}  }^T}$, ${\mathbf{A}  } = {\mathbf{G}  }{{\mathbf{P}  }^{ - 1/2}}$.
\end{theorem}

In summary, by analyzing the relationship between K-Means and manifold learning, we integrate them into a unified clustering model. The construction of manifold structure $\mathbf{M}$ from labels $\mathbf{G}$ not only maintains the consistency between manifold structure and labels but also leverages both global and local information of the data for clustering, thereby enhancing the accuracy of clustering.

The right side of Eq.~(\ref{theorem_manifold}) is a calculation based on the sum of distances between points, in order to facilitate the calculation, we turn the distance minimization problem into the matrix trace minimization problem.
\begin{equation}\label{zhuanhuan}
	\begin{aligned}
		\mathop{\mathrm{min}}\limits_{\mathbf{G}}
		\sum\limits_{i,l} {d_{il}  \left\langle {{{({\mathbf{A}^i})}  },{{({\mathbf{A}^l})}  }}   \right\rangle}
		&=\mathop{\mathrm{min}}\limits_{\mathbf{G}} \mathrm{tr}(\mathbf{A}^T \mathbf{D} \mathbf{A}) \\
		&= \mathop{\mathrm{min}}\limits_{\mathbf{G}} \mathrm{tr}(\mathbf{G}^T \mathbf{D} \mathbf{G} \mathbf{P}^{-1}) \\
	\end{aligned}
\end{equation}
where $\mathbf{G} \in \mathbb{R}^{N \times K}$ denotes the label matrix, and $\mathbf{D} \in \mathbb{R}^{N \times N}$ represents the distance matrix.

The distance matrix $\mathbf{D}$ here is calculated based on different distance measurement methods, among which the square Euclidean distance is the most commonly used metric. It measures the square of the absolute distance between points in Euclidean space. The distance between the samples $i$-th and $l$-th can be expressed as: 
$d_{ED}({il})   = \left\| {\mathbf{x}_i- \mathbf{x}_l} \right\|_F^2$.

\section{Methodology}
\subsection{Our Model}


However, the model (\ref{zhuanhuan}) is difficult to solve,
so to optimize the model, we introduce Theorem \ref{theorem4}.

\begin{theorem}\label{theorem4}
	Given $n_1+\ldots+n_K=N$, the maximum value of Eq.(\ref{the4})  is reached  when $n_1=\ldots=n_K=\frac{N}{K}$. In this case, $\textbf{G}$ is discrete and exhibits a balanced class distribution.
	\begin{equation}\label{the4}
		\begin{aligned}
			&\mathop {{\rm{max}}}\limits_{{{\textbf{G}} }}  { \left\| { \textbf{G}} \right\|_{sp}^p} 
			\quad 
			s.t.{\kern 1pt} {\textbf{G}\geq0},\textbf{G}\textbf{1}=\textbf{1}
		\end{aligned}
	\end{equation}
\end{theorem}

Theorem \ref{theorem4} tells us that Eq.(\ref{the4})  can achieve an approximate class equilibrium.
When the optimal solution is obtained, $(\mathbf{G}^T\mathbf{G})^{1/2} = \frac{N}{K}\mathbf{I}$ .
Therefore, we convert the model (\ref{zhuanhuan}) 
into a continuous model (\ref{model2}) 
\begin{equation}\label{model2}
	\begin{aligned}
		\mathop {{\rm{min}}}\limits_{{{\mathbf{G}} }} {\rm{tr}}({\mathbf{G}^T}{{\bf{D}} }{\mathbf{G}  })-\alpha { \left\| { \mathbf{G}} \right\|_{sp}^p} 
		\quad 	s.t.{\kern 1pt} {\mathbf{G}\geq0},\mathbf{G}\mathbf{1}=\mathbf{1}
	\end{aligned}
\end{equation}
when model (\ref{model2}) obtains the optimal solution, $\mathbf{G}$ is discrete and each class is balanced.

\subsection{Optimization}

The Schatten p-norm involves the sum of the singular values of the matrix and is generally not a smooth function, so when dealing with model (\ref{model2}) that contains the Schatten p-norm, direct optimization using gradient descent can become complex and difficult. So we introduce Theorem~\ref{theorem5}.

\begin{theorem}\label{theorem5}
	The derivative of $\| \mathbf{G} \|_{sp}^p$ with respect to $\mathbf{G}$  is:
	\begin{equation}\label{H1}
		{\bf{F}}= \frac{{\partial {{\left\| {\bf{G}} \right\|}_{sp}^p}}}{{\partial {\bf{G}}}}=p \mathbf{U} \mathbf{\Sigma}^{-1} |\mathbf{\Sigma} |^p \mathbf{V}^{\textrm{T}}
	\end{equation}
	where $\mathbf{G}= \mathbf{U} \mathbf{\Sigma} \mathbf{V}^{\textrm{T}}$, $\mathbf{\Sigma}^{-1}$is the Moore-Penrose pseudo-inverse of $\mathbf{\Sigma}$. 
\end{theorem}

\begin{proof} 
	Let $\|\mathbf{G}\|_{sp}^p = \sum_{i=1} \sigma_i^p(\mathbf{G})$ be the p power of the Schatten p-norm of matrix $\mathbf{G}$, where	$\sigma_i(\mathbf{G})$ denote the $i$-th largest singular value of $\mathbf{G}$.
	We express the singular values in terms of the eigenvalues of $\mathbf{G}^T\mathbf{G}$. Specifically,
	\begin{equation}\label{SP2}
		\begin{aligned}
			\sigma_i(\mathbf{G}) =  {\tau_i\sqrt{\mathbf{G}^T\mathbf{G}}} = {\tau_i{(\mathbf{A}^\frac{1}{2}) }}
		\end{aligned}
	\end{equation}
	where $\tau_i(\mathbf{G}^T\mathbf{G})$ denote  $i$-th largest eigenvalue of $\mathbf{G}^T\mathbf{G}$, $\mathbf{A}=\mathbf{G}^T\mathbf{G}$.
	
	Thus, the $p$-th power of the singular values can be expressed as:
	\begin{equation}\label{SP3}
		\begin{aligned}
			\sigma_i^p(\mathbf{G}) =  {\tau_i{(\mathbf{A}^\frac{p}{2}) }}
		\end{aligned}
	\end{equation}
	
	The Schatten p-norm of $\mathbf{G}$ is then:
	\begin{equation}\label{sp4}
		\begin{aligned}
			\|\mathbf{G}\|_{sp}^p &=  tr(\mathbf{A}^\frac{p}{2})=tr(({\mathbf{G}^T\mathbf{G}})^\frac{p}{2})\\
			&=tr((\mathbf{V}^T\mathbf{V}\mathbf{\Sigma}^2 )^\frac{p}{2})
			=tr(\left| \mathbf{\Sigma} \right| ^p)
		\end{aligned}
	\end{equation}
	where $\mathbf{\Sigma}$ denotes the diagonal matrix of singular values.
	
	Next, the gradient of the Schatten p-norm with respect to $\mathbf{G}$ is given by:
	\begin{equation}\label{S}
		\frac{\partial \|\mathbf{G}\|_{sp}^p}{\partial \mathbf{G}} = \frac{\partial tr(\left| \mathbf{\Sigma} \right| ^p)}{\partial \mathbf{G}}=\frac{tr(\partial  \left| \mathbf{\Sigma} \right| ^p)}{\partial \mathbf{G}}
	\end{equation}
	
	The subdifferential of $\mathbf{G}$ is:
	\begin{equation}\label{S1}
		\frac{\partial \left| \mathbf{\Sigma} \right|}{\partial \mathbf{G}}
		=\left| \mathbf{\Sigma} \right| \mathbf{\Sigma}^{-1} \frac{\partial \mathbf{\Sigma}}{\partial \mathbf{G}}
	\end{equation}
	
	and thus, the gradient of $|\mathbf{\Sigma}|^p$ with respect to $\mathbf{G}$ is:
	\begin{equation}\label{S2}
		\frac{\partial \left| \mathbf{\Sigma} \right|^p}{\partial \mathbf{G}}=p\left| \mathbf{\Sigma} \right|^{p-1}\frac{\partial \left| \mathbf{\Sigma} \right|}{{\partial \mathbf{G}}}=\frac{p\left| \mathbf{\Sigma} \right|^{p-1}\partial \left| \mathbf{\Sigma} \right|}{\partial \mathbf{G}}
	\end{equation}
	
	Substituting Eq.~(\ref{S1}) and Eq.~(\ref{S2}) back into Eq.~(\ref{S})
	\begin{equation}\label{S3}
		\frac{\partial \|\mathbf{G}\|_{sp}^p}{\partial \mathbf{G}} = \frac{tr(p\left| \mathbf{\Sigma} \right|^{p-1}\left| \mathbf{\Sigma} \right|\mathbf{\Sigma}^{-1}\partial \mathbf{\Sigma}
			)}	{\partial \mathbf{G}}
	\end{equation}
	
	According to $\mathbf{G}=\mathbf{U}\mathbf{\Sigma}\mathbf{V}^T$,
	\begin{equation}\label{S4}
		\begin{aligned}
			{\partial \mathbf{G}}=\partial \mathbf{U}\mathbf{\Sigma}\mathbf{V}^T+\mathbf{U}\partial \mathbf{\Sigma}\mathbf{V}^T+\mathbf{U}\mathbf{\Sigma}\partial \mathbf{V}^T\\
		\end{aligned}
	\end{equation}
	
	Rearranging the second term:
	\begin{equation}\label{S41}
		\begin{aligned}
			\mathbf{U}\partial \mathbf{\Sigma}\mathbf{V}^T={\partial \mathbf{G}}-\partial \mathbf{U}\mathbf{\Sigma}\mathbf{V}^T-\mathbf{U}\mathbf{\Sigma}\partial \mathbf{V}^T\\
		\end{aligned}
	\end{equation}
	
	Next, we multiply both sides of the equation by $\mathbf{U}^T$ and $\mathbf{V}$  to obtain:
	\begin{equation}\label{S5}
		\begin{aligned}
			\mathbf{U}^T\mathbf{U}\partial \mathbf{\Sigma}\mathbf{V}^T\mathbf{V}&={\mathbf{U}^T\partial \mathbf{G}}\mathbf{V}-\mathbf{U}^T\partial \mathbf{U}\mathbf{\Sigma}\mathbf{V}^T\mathbf{V}\\
			&-\mathbf{U}^T\mathbf{U}\mathbf{\Sigma}\partial \mathbf{V}^T\mathbf{V}\\
		\end{aligned}
	\end{equation}
	
	This simplifies to:
	\begin{equation}\label{S51}
		\begin{aligned}
			\partial \mathbf{\Sigma}&={\mathbf{U}^T\partial \mathbf{G}}\mathbf{V}-\mathbf{U}^T\partial \mathbf{U}\mathbf{\Sigma}-\mathbf{\Sigma}\partial \mathbf{V}^T\mathbf{V}\\
		\end{aligned}
	\end{equation}
	
	Consider the identity:
	\begin{equation}\label{S6}
		\begin{aligned}
			0=\partial \mathbf{I}=\partial(\mathbf{U}^T\mathbf{U})=\partial\mathbf{U}^T\mathbf{U}+\mathbf{U}^T\partial\mathbf{U}
		\end{aligned}
	\end{equation}
	
	From this, we can get:
	\begin{equation}\label{S7}
		\begin{aligned}
			tr(\mathbf{U}^T\partial\mathbf{U}\mathbf{\Sigma})&=tr((\mathbf{U}^T\partial\mathbf{U}\mathbf{\Sigma})^T)\\
			&=tr(\mathbf{\Sigma}^T\partial\mathbf{U}^T\mathbf{U})\\
			&=-tr(\mathbf{\Sigma}\mathbf{U}^T\partial\mathbf{U})\\
			&=-tr(\mathbf{U}^T\partial\mathbf{U}\mathbf{\Sigma})\\
			&=-tr(\mathbf{U}^T\partial\mathbf{U}\mathbf{\Sigma})\\
			&=0
		\end{aligned}
	\end{equation}
	
	
	In a similar way,
	\begin{equation}\label{S8}
		\begin{aligned} tr(\mathbf{\Sigma}\partial\mathbf{V}^T\mathbf{V})=0
		\end{aligned}
	\end{equation}
	
	Combine  Eq.~(\ref{S51}) to  Eq.~(\ref{S8}), we have
	\begin{equation}\label{S9}
		\begin{aligned} tr(\partial\mathbf{\Sigma})=tr(\mathbf{U}^T\partial\mathbf{G}\mathbf{V})
		\end{aligned}
	\end{equation}
	
	Substituting Eq.~(\ref{S9}) back into Eq.~(\ref{S3}),
	\begin{equation}\label{S10}
		\begin{aligned}
			\frac{\partial \|\mathbf{G}\|_{sp}^p}{\partial \mathbf{G}} &= \frac{tr(p\left| \mathbf{\Sigma} \right|^{p-1}\left| \mathbf{\Sigma} \right|\mathbf{\Sigma}^{-1}\mathbf{U}^T\partial\mathbf{G}\mathbf{V}
				)}	{\partial \mathbf{G}}\\
			&=\frac{tr(\mathbf{V}p\left| \mathbf{\Sigma} \right|^{p-1}\left| \mathbf{\Sigma} \right|\mathbf{\Sigma}^{-1}\mathbf{U}^T\partial\mathbf{G}
				)}	{\partial \mathbf{G}}\\
			&=(\mathbf{V}p\left| \mathbf{\Sigma} \right|^{p-1}\left| \mathbf{\Sigma} \right|\mathbf{\Sigma}^{-1}\mathbf{U}^T)^T\\
			&=(\mathbf{V}p\left| \mathbf{\Sigma} \right|^{p}\mathbf{\Sigma}^{-1}\mathbf{U}^T)^T\\
			&=p\mathbf{U}\mathbf{\Sigma}^{-1} \left| \mathbf{\Sigma} \right|^{p}		
			\mathbf{V}^T\\
		\end{aligned}
	\end{equation}
\end{proof}

So we  approximate  Eq.(\ref{model2}) to Eq.(\ref{au}), $\mathbf{G}$ is updated by solving the following problem 
\begin{equation}\label{au}
	\begin{aligned}
		\min_{\mathbf{G}\mathbf{1}=\mathbf{1}, \mathbf{G} \geqslant \mathbf{0} } 
		tr(\mathbf{G^{\rm T}DG})   - \alpha tr(\mathbf{F^{\rm T}G})
	\end{aligned}
\end{equation}

The solution for $\mathbf{g}^i$ is
\begin{equation}\label{Solveyi}
	\begin{aligned}
		{g_{ib}} = \left\{ \begin{array}{l}
			1, b = \mathop {\arg \min }\limits_j {(2{{\bf{G}}^{\rm{T}}}{{\bf{d}}_i} - \alpha {({{\bf{f}}^i})^{\rm{T}}})_j}\\
			0, {\rm{otherwise}}.
		\end{array} \right.
	\end{aligned}
\end{equation}

	%
%
%
%

\begin{figure*}[!ht]
	\begin{minipage}[t]{0.45\textwidth}
		\centering
		\includegraphics[scale=0.45]{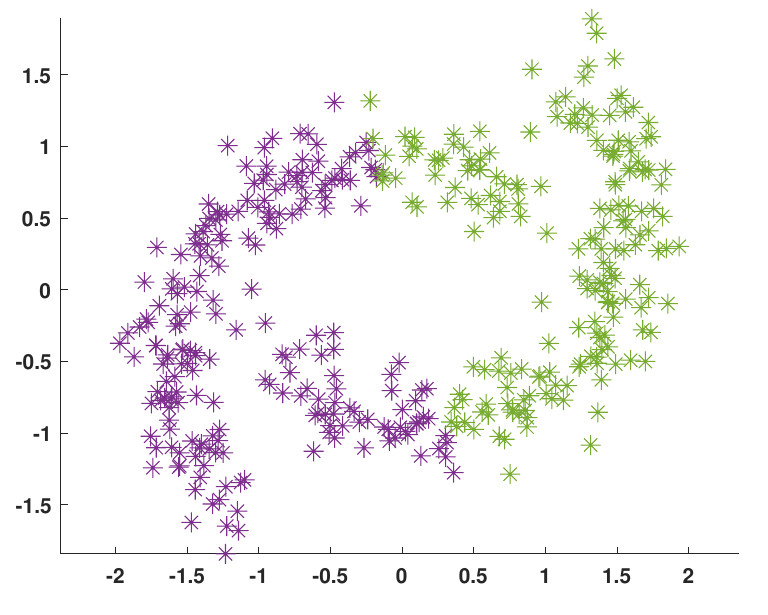}
		\subcaption*{(a)K-means}
	\end{minipage}
	\begin{minipage}[t]{0.45\textwidth}
		\centering
		\includegraphics[scale=0.45]{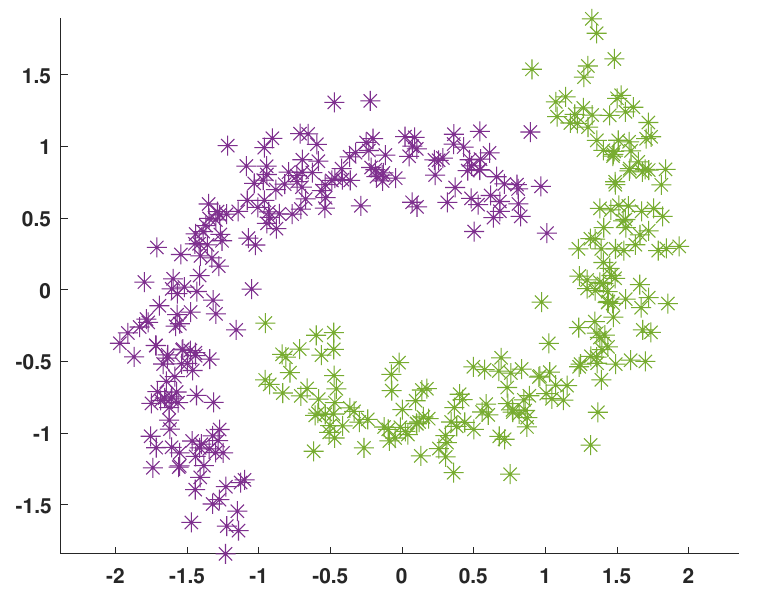}
		\subcaption*{(a)Our}
	\end{minipage}
	\caption{Visual comparison of two-moon under different algorithms.
		Different colors represent different classes.}\label{twospiral_ours}
\end{figure*}

\section{Experiments}

\subsection{Experiments of Artificial Dataset}
In this section, we conduct experiments to validate the effectiveness of our approach for clustering nonlinearly separable data. For this purpose, we create two synthetic datasets. First, we generate a  two-moon  dataset, which consists of 400 samples, each represented in two-dimensional coordinates. 
In addition, we construct a synthetic dataset called the two-spiral dataset, which is also characterized by nonlinear separable clustering properties. 
%
%

\subsection{Experimental on Benchmark Datasets}

\begin{table*}[ht]
	\caption{The clustering performances  on the  JAFFE, ORL datasets.}\label{result1}
	\centering
	\begin{center}
		\resizebox{1\columnwidth}{!}{
			\begin{tabular}{l | l l l l l l| l l l l l l}
				\toprule
				Datasets &\multicolumn{6}{c|}{JAFFE} &\multicolumn{6}{c}{ORL}  
				\\
				\midrule
				
				Methods & ACC & NMI & Purity &Precision  &F-score&ARI& ACC & NMI & Purity &Precision  &F-score&ARI \\
				\midrule
				CAN~\cite{can14kdd}&0.9108    &0.9175    &0.9108    &0.8236       &0.8536    &0.8375			
				&0.5350 &0.7173 &0.6225&0.1504 &0.2462&0.2173 \\
				LCSOG~\cite{LCSOG}&\textbf{0.9671}&\textbf{0.9623}&\textbf{0.9671}&0.9323&0.9335&0.9291&0.6075&0.7839&0.6725&0.3320&0.4385&0.4212\\
				{K}-means~\cite{hartigan1979algorithm} &0.7085&0.8010&0.7455&0.5970&0.6846&0.6450&0.5198&0.7234&0.5705&0.3037&0.3722&0.3543 \\
				KKM~\cite{KK} &0.8028    &0.8246    &0.8263    &0.6656        &0.7238    &0.6917&0.5425    &0.7440    &0.5800    &0.3658       &0.4168    &0.4014\\
				RKM~\cite{RMK}   &0.8310	&0.8159	&0.8310	&0.7374&0.7369&0.7090&0.5000	&0.7143	&0.5200	&0.3756&0.3756&0.3611
				\\
				CDKM~\cite{CDKM}   &0.7451&0.8246&0.7812&0.6381	&0.7180	&0.6828
				&0.5507	&0.7529	&0.6090	&0.3692		&0.4340	&0.4185
				
				\\
				K-sum~\cite{Ksum}	&0.8789	&0.8764	&0.8789	&0.8034&0.8130&0.7929&0.6337	&0.7940	&0.6562&0.5276&0.5298&0.5189		\\
				K-sumx~\cite{Ksum}   &0.8930	&0.9013	&0.8977	&0.8319&0.8397&0.8225&0.5877	&0.7693	&0.6060&0.4742&0.4786&0.4665		\\
				{Our-ED}	&\textbf{0.9671}    &0.9548	&\textbf{0.9671}	&\textbf{0.9328}	&\textbf{0.9378}	&\textbf{0.9311}&0.6575&0.8042&0.6750&0.5527    &0.5574   &0.5471			
				\\
				
				\bottomrule
		\end{tabular}}
	\end{center}
\end{table*}

\subsubsection{Results}
After conducting experiments, we get the corresponding measurement results. These results are presented in a detailed comparison across Tables \ref{result1} . 
Based on this analysis, we can draw the following conclusions:

First, our model can adapt to different types of distance matrices. In the experiment, we compare Euclidean distance  (Our-ED) and KNN distance (Our-KNN), and the results demonstrate that KNN distance significantly enhances clustering performance. This improvement is attributed to the KNN distance's ability to effectively capture the nonlinear relationships between data points, thereby achieving more accurate clustering results. Consequently, our model can  handle complex datasets by  KNN distance metric.
Then,  we observe that algorithms dependent on the centroid matrix---{K}-means, RKM---perform less effectively on the baseline dataset compared to the K-sum and K-sumx. Specifically, K-sum and K-sumx algorithms integrate spectral clustering and {K}-means in the clustering process, avoiding the centroid matrix. This integration enables them to exhibit higher accuracy when dealing with complex datasets.
We fuse manifold learning and k-means into a unified model with a clever combination of Schatten p-norm. It can not only capture the manifold structure of the data, but also maintain the consistency of the manifold structure and  label, and naturally maintain the class balance, showing better clustering performance

\subsubsection{Parameters setting and analysis}

For the JAFFE, the optimal clustering effect is achieved with $\alpha = 1000$ and $C = 212$.
On the ORL, the model demonstrates the best clustering performance with $\alpha = 8500000$ and $C = 10$.
These findings underscore the importance of carefully selecting the Schatten p-norm weight $\alpha$ and the KNN parameter $C$
to significantly enhance the model's clustering effectiveness.

\subsubsection{Discussion of the value of p} 
To investigate the effect of the parameter $p$ in the Schatten p-norm on clustering results, we conducted experiments using the our-KNN algorithm on the JAFFE, ORL datasets, while keeping other parameters fixed. The value of $p$ ranged from 0.1 to 2.
The results show that for the JAFFE dataset, the best clustering performance is achieved when $p = 2$; for ORL, $p = 1.9$ yields the optimal performance.

\section{Conclusion}
This paper presents a new manifold clustering framework, which fuses manifolds and distances.
The manifold structure is constructed by labels, and the labels are optimized by clustering on the manifold, thus ensuring the consistency of manifold structure and labels.
Further,  we effectively maintain the class balance of the data by maximizing the Schatten p-norm of the label.
Additionally, by flexibly applying different distance matrices,  our model performs well on nonlinearly separable data, 
significantly enhancing the clustering performance. A large number of experiments fully verify the  effectiveness of this method.

\section*{Acknowledgment}
The authors would like to thank the editors and reviewers for their helpful comments and suggestions.

\bibliographystyle{elsarticle-num}
\bibliography{ref}

\end{document}